\pdfoutput=1
\documentclass{article} % For LaTeX2e
\usepackage{times}
\usepackage{geometry}

\usepackage[utf8]{inputenc} % allow utf-8 input
\usepackage[T1]{fontenc}    % use 8-bit T1 fonts
\usepackage[colorlinks,citecolor=blue]{hyperref}       % hyperlinks
\usepackage{url}            % simple URL typesetting
\usepackage{booktabs}       % professional-quality tables
\usepackage{amsfonts}       % blackboard math symbols
\usepackage{nicefrac}       % compact symbols for 1/2, etc.
\usepackage{microtype}      % microtypography

\usepackage{url}
\usepackage{authblk}

\usepackage{graphicx}
\usepackage{multirow}
\usepackage{floatrow}
\newfloatcommand{capbtabbox}{table}[][\FBwidth]

\PassOptionsToPackage{numbers}{natbib}
\usepackage{natbib}

\usepackage{enumerate}
\usepackage{amsthm,amsmath,amssymb}
\usepackage{amsfonts,dsfont}
\usepackage{nicefrac}
\usepackage{microtype}
\usepackage{mathtools}

\usepackage{caption}
\usepackage{subcaption}
\usepackage{enumitem}
\usepackage{algorithm}
\usepackage{algorithmic}
\usepackage[normalem]{ulem}
\usepackage{amssymb}
\usepackage{multicol}
\usepackage{adjustbox}
\usepackage{multirow}
\usepackage{color}
\usepackage{xspace}
\usepackage{mkolar_definitions}

\usepackage{comment}
\newcount\Comments  % 0 suppresses notes to selves in text
\definecolor{darkgreen}{rgb}{0,0.5,0}
\definecolor{darkred}{rgb}{0.7,0,0}
\definecolor{teal}{rgb}{0.3,0.8,0.8}
\newcommand{\kibitz}[2]{\ifnum\Comments=1\colorbox{magenta!20}{\parbox{\linewidth}{\textcolor{#1}{\textsf{\footnotesize #2}}}}\fi}

\newenvironment{packed_item}{
	\begin{itemize}[leftmargin=*]
		\setlength{\itemsep}{1pt}
		\setlength{\parskip}{0pt}
		\setlength{\parsep}{0pt}
	}{\end{itemize}}

\newcounter{qcounter}
	{\end{list}}

\newtheorem{assumption}{Assumption}
\newcommand{\alg}{\textsc{Perch}\xspace}
\newcommand{\algc}{\textsc{Perch-C}\xspace}
\newcommand{\algb}{\textsc{Perch-B}\xspace}
\newcommand{\algbc}{\textsc{Perch-BC}\xspace}
\newcommand{\cstar}{\ensuremath{\mathcal{C}^\star}}
\newcommand{\chat}{\ensuremath{\widehat{\mathcal{C}}}}
\newcommand{\pstar}{\ensuremath{\mathcal{P}^\star}}
\newcommand{\phat}{\ensuremath{\widehat{\mathcal{P}}}}
\newcommand{\lvs}[1]{\ensuremath{\texttt{lvs}(#1)}}

\newcommand{\bal}[1]{\ensuremath{\texttt{bal}(#1)}}
\newcommand{\km}{\emph{K}-means}

\begin{document}

	\title{An Online Hierarchical Algorithm for Extreme Clustering}

	\author{
		Ari Kobren
		\thanks{The first two authors contributed equally.}}
	\author{
		Nicholas Monath$^*$}
	\author{
		Akshay Krishnamurthy}

	\author{
		Andrew McCallum	}

	\affil{College of Information and Computer Sciences \\ University of Massachusetts Amherst}
	\affil{\{akobren,nmonath,akshay,mccallum\}@cs.umass.edu}
	\maketitle

	\begin{abstract}
Many modern clustering methods scale well to a large number of data
items, $N$, but not to a large number of clusters, $K$.  This paper
introduces \alg, a new non-greedy algorithm for online hierarchical
clustering that scales to both massive $N$ and $K$---a problem setting
we term {\it extreme clustering}.  Our algorithm efficiently routes
new data points to the leaves of an incrementally-built tree.
Motivated by the desire for both accuracy and speed, our approach
performs tree rotations for the sake of enhancing subtree purity
and encouraging balancedness.  We prove that, under a natural
separability assumption, our non-greedy algorithm will produce trees
with perfect dendrogram purity regardless of online data arrival
order.  Our experiments demonstrate that \alg constructs more accurate
trees than other tree-building clustering algorithms and scales well
with both $N$ and $K$, achieving a higher quality clustering than the
strongest flat clustering competitor in nearly half the time.

%Clustering massive data sets has been made possible by online
%clustering algorithms. However, when a large data set also has a very
%large number of clusters, the same algorithms often do not scale.  We
%present a new online clustering algorithm that is designed to scale
%with both the number of data points and the number of clusters. Our
%algorithm incrementally builds a tree with data points at its leaves
%to facilitate fast routing. The key to our approach is a purity
%enhancing tree rotation that is applied efficiently and recursively to
%help maintain the tree's purity. We prove that when the data set being
%clustered exhibits a certain form of $separability$, our algorithm
%will construct a tree with provably optimal purity.  Our experiments
%demonstrate that even when $separability$ is violated, our algorithm
%is up to {\color{red}X.xx} faster and up to {\color{red}XX.x\%} more
%accurate on large real world data than other state of the art
%approaches. \nick{I think we should define ``extreme clustering'' here
%if it is staying in the title.}
\end{abstract}

	\maketitle

	\section{Introduction}
Clustering algorithms are a crucial component of any data scientist's
toolbox with applications ranging from identifying themes in large
text corpora~\cite{brown1992class}, to finding functionally similar
genes~\cite{eisen1998cluster}, to visualization, pre-processing, and
dimensionality reduction~\cite{gorban2008principal}.  As such, a
number of clustering algorithms have been developed and studied by the
statistics, machine learning, and theoretical computer science
communities. These algorithms and analyses target a variety of
scenarios, including large-scale, online, or streaming
settings~\cite{zhang1996birch,ackermann2012streamkm++}, clustering
with distribution shift~\cite{aggarwal2003framework}, and many more.

Modern clustering applications require algorithms that scale
gracefully with dataset size and complexity. In clustering, data set
size is measured by the number of points $N$ and their dimensionality
$d$, while the number of clusters, $K$, serves as a measure of
complexity. While several existing algorithms can cope with large
datasets, very few adequately handle datasets with many clusters.
We call problem instances with large $N$ and large
$K$ \emph{extreme clustering problems}--a phrase inspired by work in
extreme classification~\cite{choromanska2015logarithmic}.

Extreme clustering problems are increasingly prevalent. For example,
in entity resolution, record linkage and deduplication, the number of
clusters (i.e., entities) increases with dataset
size~\cite{betancourt2016flexible} and can be in the millions.
Similarly, the number of communities in
typical real world networks tends to follow a power-law
distribution~\cite{clauset2004finding} that also increases with
network size. Although used primarily in classification tasks, the
ImageNet dataset also describes precisely this large $N$, large $K$
regime ($N\approx 14$M,$K\approx 21$K)~\cite{deng2009imagenet,
  ILSVRC15}.

This paper presents a new online clustering algorithm, called \alg, that scales
mildly with both $N$ and $K$ and thus addresses the extreme clustering
setting.  Our algorithm constructs a tree structure over the data
points in an incremental fashion by routing incoming points to the
leaves, growing the tree, and maintaining its quality via simple
\emph{rotation} operations. The tree structure enables efficient
(often logarithmic time) search that scales to large datasets,
while simultaneously providing a rich data structure from which
multiple clusterings at various resolutions can be extracted.  The
rotations provide an efficient mechanism for our algorithm
to recover from mistakes that arise with greedy incremental
clustering procedures.

Operating under a simple separability assumption about the data, we
prove that our algorithm constructs a tree with perfect
\emph{dendrogram purity} regardless of the number of data points and
without knowledge of the number of clusters. This analysis relies
crucially on a recursive rotation procedure employed by our algorithm.
For scalability, we introduce another flavor of rotations that encourage
balancedness, and an approximation that enables faster point
insertions.  We also develop a leaf \emph{collapsing} mode of our algorithm
that can operate in memory-limited settings, when the dataset does not
fit in main memory.

We empirically demonstrate that our algorithm is both accurate and
efficient for a variety of real world data sets. In comparison to
other tree-building algorithms (that are both online and multipass),
\alg achieves the highest dendrogram purity in addition to being efficient.
When compared to flat clustering algorithms where the number of
clusters is given by an oracle, \alg with a pruning heuristic
outperforms or is competitive with all other scalable algorithms.
In both comparisons to flat and tree building algorithms,
\alg scales best with the number of clusters $K$.

	\section{The Clustering Problem}
\label{sec:ct}
In a \emph{clustering problem} we are given a dataset $X =
\{x_i\}_{i=1}^N$ of \emph{points}. The goal is to partition the
dataset into a set of disjoint subsets (i.e., clusters),
$\mathcal{C}$, such that the union of all subsets covers the
dataset. Such a set of subsets is called a \emph{clustering}.  A
high quality clustering is one in which the points in any particular
subset are more similar to each other than the points in other
subsets.

A clustering, $\mathcal{C}$, can be represented as a map from points
to clusters identities, $\mathcal{C}: X \to \{1,\ldots,K\}$.  However,
structures that encode more fine-grained information also exist. For
example, prior works construct a \emph{cluster tree} over the dataset
to compactly encode multiple alternative \emph{tree-consistent
  partitions}, or clusterings \cite{heller2005bayesian,
  blundell2011discovering}. 

\begin{definition}[Cluster tree~\cite{krishnamurthy2012efficient}]
A binary \textbf{cluster tree} $\mathcal{T}$ on a dataset
$\{x_i\}_{i=1}^N$ is a collection of subsets such that $C_0 \triangleq
\{x_i\}_{i=1}^N \in \mathcal{T}$ and for each $C_i,C_j \in
\mathcal{T}$ either $C_i \subset C_j$, $C_j \subset C_i$ or $C_i \cap
C_j = \emptyset$.  For any $C \in \mathcal{T}$, if $\exists C' \in
\mathcal{T}$ with $C' \subset C$, then there exists two $C_L,C_R\in
\mathcal{T}$ that partition $C$.
\end{definition}

Given a \emph{cluster tree}, each internal node can be associated with
a cluster that includes its descendant points. A tree-consistent
partition is a subset of the nodes in the tree whose associated
clusters partition the dataset, and hence is a valid clustering.

Cluster trees afford multiple advantages in addition to their
representational power.  For example, when building a cluster tree it
is typically unnecessary to specify the number of target clusters
(which, in virtually all real-world problems, is unknown). Cluster
trees also provide the opportunity for efficient cluster assignment
and search, which is particularly important for large datasets with
many clusters. In such problems, $O(K)$ search required by classical
methods can be prohibitive, while a top down traversal of a cluster tree
could offer $O(\log(K))$ search, which is exponentially faster.

Evaluating a cluster tree is more complex than evaluating a \emph{flat
  clustering}. Assuming that there exists a ground truth clustering
$\cstar = \{C_k^\star\}_{i=1}^K$ into $K$ clusters, it is common to
measure the quality of a cluster tree based on a single clustering
extracted from the tree. We follow Heller et. al and adopt a more
holistic measure of the tree quality, known as \emph{dendrogram
  purity}~\cite{heller2005bayesian}.  Define
\begin{align*}
  \mathcal{P}^\star = \{(x_i,x_j) \mid \cstar(x_i) = \cstar(x_j)\}
\end{align*}
to be the pairs of points that are clustered together in the ground
truth. Then, dendrogram purity is defined as follows:

\begin{definition}[Dendrogram Purity]
Given a cluster tree $\mathcal{T}$ over a dataset $X
=\{x_i\}_{i=1}^N$, and a true clustering $\cstar$, the \textbf{dendrogram
purity} of $\mathcal{T}$ is
\begin{align*}
\texttt{DP}(\mathcal{T}) = \frac{1}{|\mathcal{P}^\star|}\sum_{k=1}^K\sum_{x_i,x_j \in \cstar_k}\texttt{pur}(\lvs{\texttt{LCA}(x_i,x_j)}, \cstar_k)
\end{align*}
where $\texttt{LCA}(x_i,x_j)$ is the least common ancestor of $x_i$
and $x_j$ in $\mathcal{T}$, $\lvs{z}\subset X$ is the set of
leaves for any internal node $z$ in $\mathcal{T}$, and
\texttt{pur}$(S_1,S_2) = |S_1 \cap S_2|/|S_1|$.
\end{definition}

In words, the dendrogram purity of a tree with respect to a ground
truth clustering is the expectation of the following random process:
(1) sample two points, $x_i,\ x_j$, uniformly at random from the pairs in the ground
truth (and thus $\cstar(x_i) = \cstar(x_j)$), (2) compute their least common ancestor in
$\Tcal$ and the cluster (i.e. descendant leaves) of that internal node, (3)
compute the fraction of points from this cluster that also belong to $\cstar(x_i)$.
For large-scale problems, we use Monte Carlo approximations of dendrogram purity.  More intuitively,
dendrogram purity obviates the (often challenging) task of extracting
the best tree-consistent partition, while still providing a meaningful
measure of overlap with the ground truth flat clustering.

\iffalse
\subsection{Cluster Separation}

\label{subsec:delsep}
A high quality cluster tree is one that exhibits dendrogram purity
close to 1.0 (optimal). In the following section we derive an online
algorithm that achieves provably optimal dendrogram purity assuming
that the data being clustered is \emph{separable}.

In other words, all data records that are members of the same cluster
must be closer to each other than to any member of any other cluster.
\fi

	\section{Tree Construction}
Our work is focused on instances of the clustering problem in which
the size of the dataset $N$ \emph{and} the number of clusters $K$ are
both very large.  In light of their advantages with respect to
efficiency and representation (Section \ref{sec:ct}), our method
builds a cluster tree over data points. We are also interested in the
\emph{online} problem setting--in which data points arrive one at a
time--because this resembles real-world scenarios in which new data is
constantly being created. Well-known clustering methods based on
cluster trees, like hierarchical agglomerative clustering, are
often not online; there exist a few online cluster tree approaches,
most notably BIRCH~\cite{zhang1996birch}, but empirical results show that BIRCH
typically constructs worse clusterings than most
competitors~\cite{o2002streaming, ackermann2012streamkm++}.

The following subsections describe and analyze several fundamental
components of our algorithm, which constructs a cluster tree in an
online fashion. The data points are assumed to reside in Euclidean
space: $\{x_i\}_{i=1}^N \subset \mathbb{R}^d$.  We make no assumptions
on the order in which the points are processed.

\subsection{Preliminaries}
\label{subsec:separability}
In the clustering literature, it is common to make various
\emph{separability assumptions} about the data being clustered. As
one example, a set of points is $\epsilon$-separated for \km\ if the ratio of
the clustering cost with $K$ clusters to the cost with $K-1$ clusters
is less than $\epsilon^2$ \cite{ostrovsky2006effectiveness}. While
assumptions like these generally do not hold in practice, they
motivate the derivation of several powerful and justifiable
algorithms. To derive our algorithm, we make a strong separability
assumption under $\cstar$.

\begin{assumption}[Separability]
\label{def:delsep}
A data set $X$ is \textbf{separable} with respect to a clustering $\cstar$ if
\begin{align*}
    \max_{(x,y) \in \mathcal{P}^\star}\|x-y\| < \min_{(x',y') \notin \mathcal{P}^\star} \|x'-y'\|,
\end{align*}
where $\|\cdot\|$ is the Euclidean norm.
We assume that $X$ is \emph{separable} with respect to $\cstar$.
\end{assumption}
Thus under separability, the true clustering has all within-cluster
distances smaller than any between-cluster distance.  The assumption
is quite strong, but it is not without precedent. As just one example,
the NCBI uses a form of separability in their \emph{clique}-based
approach for protein clustering~\cite{ncbi}. Moreover, separability
aligns well with existing clustering methods; for example under
separability, agglomerative methods like complete, single, and average linkage
are guaranteed to find $\cstar$, and $\cstar$ is guaranteed to contain the
unique optimum of the $K$-center cost function. Lastly, we use
separability primarily for theoretically grounding the design of our
algorithm, and we do not expect it to hold in practice. Our empirical
analysis shows that our algorithm outperforms existing methods even
without separability assumptions.

\subsection{The Greedy Algorithm and Masking}
\label{subsec:greedy}
The derivation of our algorithm begins from an attempt to remedy
issues with greedy online tree construction.  Consider the following
greedy algorithm: when processing point $x_i$, a search is preformed
to find its nearest neighbor in the current tree. The search returns a
leaf node, $\ell$, containing the nearest neighbor of $x_i$.  A
\texttt{Split} operation is performed on $\ell$ that: (1) disconnects
$\ell$ from its parent, (2) creates a new leaf $\ell'$ that stores
$x_i$, (3) creates a new internal node whose parent is $\ell$'s former
parent and with children $\ell$ and $\ell'$.

\begin{fact}
There exists a separable clustering instance in $1$-dimension with two
balanced clusters where the greedy algorithm has dendrogram purity at
most $7/8$.
\end{fact}
\begin{proof}
The construction has two clusters, one with half its points near $-1$
and half its points near $+1$, and another with all points around
$+4$, so the instance is separable. If we show one point near $-1$,
one point near $+1$, and then a point near $+4$, one child of the root
contains the latter two points, so the true clustering is
irrecoverable. To calculate the purity, notice that at least $1/2$ of
the pairs from cluster one have the root as the LCA, so their purity
is $1/2$ and the total dendrogram purity is at most $7/8$.
\end{proof}

It is easy to generalize the construction to more clusters
and higher dimension, although upper bounding the dendrogram purity
may become challenging.  The impurity in this example is a
result of the leaf at $+1$ becoming \emph{masked} when $+4$ is
inserted.

\begin{definition}[Masking]
A node $v$ with sibling $v'$ and aunt $a$ in a tree $\Tcal$ is
\textbf{masked} if there exists a point $x \in \lvs{v}$ such that
\begin{align}
\max_{y \in \lvs{v'}}\|x - y\| > \min_{z \in \lvs{a}} \|x-z\|.
\label{eq:masking_def}
\end{align}
\end{definition}
Thus, $v$ contains a point $x$ that is closer to some point in the
aunt $a$ than some point in the sibling $v'$. Intuitively, masking
happens when a point is misclustered. For example when a point
belonging to the same cluster as $v$'s leaves is sent to $a$, then $v$
becomes masked.  A direct child of the root cannot be
masked since it has no aunt.

Under separability, masking is intimately related to dendrogram
purity, as demonstrated in the following result.
\begin{fact}
If $X$ is {separated} w.r.t. \cstar and a cluster tree
$\mathcal{T}$ contains no masked nodes, then it has dendrogram purity 1.
\end{fact}
\begin{proof}
  Assume that $\Tcal$ does not have dendrogram purity 1. Then there
  exists points $x_i$ and $x_j$ in $\lvs{\Tcal}$ such that
  $\cstar(x_i) = \cstar(x_j)$ but $\lvs{\texttt{LCA}(x_i, x_j)}$
  contains a point $x_k$ in a different cluster. The least common
  ancestor has two children $v_\ell, v_r$ and $x_i,x_j$ cannot be in
  the same child, so without loss of generality we have $x_i, x_k \in
  v_l$ and $x_j \in v_r$.  Now consider $v =
  \texttt{LCA}(x_i,x_k)$ and $v$'s sibling $v'$.
  If $v'$ contains a point belonging to $\cstar(x_i)$, then the child
  of $v$ that contains $x_i$ is masked since $x_i$ is closer
  to that point than it is to $x_k$.
  If the aunt of $v$ contains a point
  belonging to $\cstar(x_i)$ then $v$ is masked for the same reason.
  If the aunt contains only points
  that do not belong to $\cstar(x_i)$ then examine $v$'s parent and
  proceed recursively. This process must terminate since we are below
  $\texttt{LCA}(x_i, x_j)$. Thus the leaf containing $x_i$ or one
  of $x_i$'s ancestors must be masked.
\end{proof}

\subsection{Masking Based Rotations}
\label{subsec:mask-rot}
Inspired by self-balancing binary search
trees~\cite{sedgewick1988algorithms}, we employ a novel
\emph{masking-based rotation} operation that alleviates purity errors
caused by masking. In the greedy algorithm, after inserting a new
point as leaf $\ell'$ with sibling $\ell$, the algorithm checks if
$\ell$ is masked. If masking is detected, a \texttt{Rotate} operation is
performed, which swaps the positions of $\ell'$ and its aunt in the
tree (See Figure~\ref{fig:rotate}). After the rotation, the algorithm
checks if $\ell$'s new sibling is masked and recursively applies
rotations up the tree. If no masking is detected at any point in the
process, the algorithm terminates.

At this point in the discussion, we check for masking exhaustively via
Equation~\eqref{eq:masking_def}. In the next section we introduce
approximations for scalability.

\begin{figure}[t]
\begin{center}
\centerline{\includegraphics[width=0.5\columnwidth]{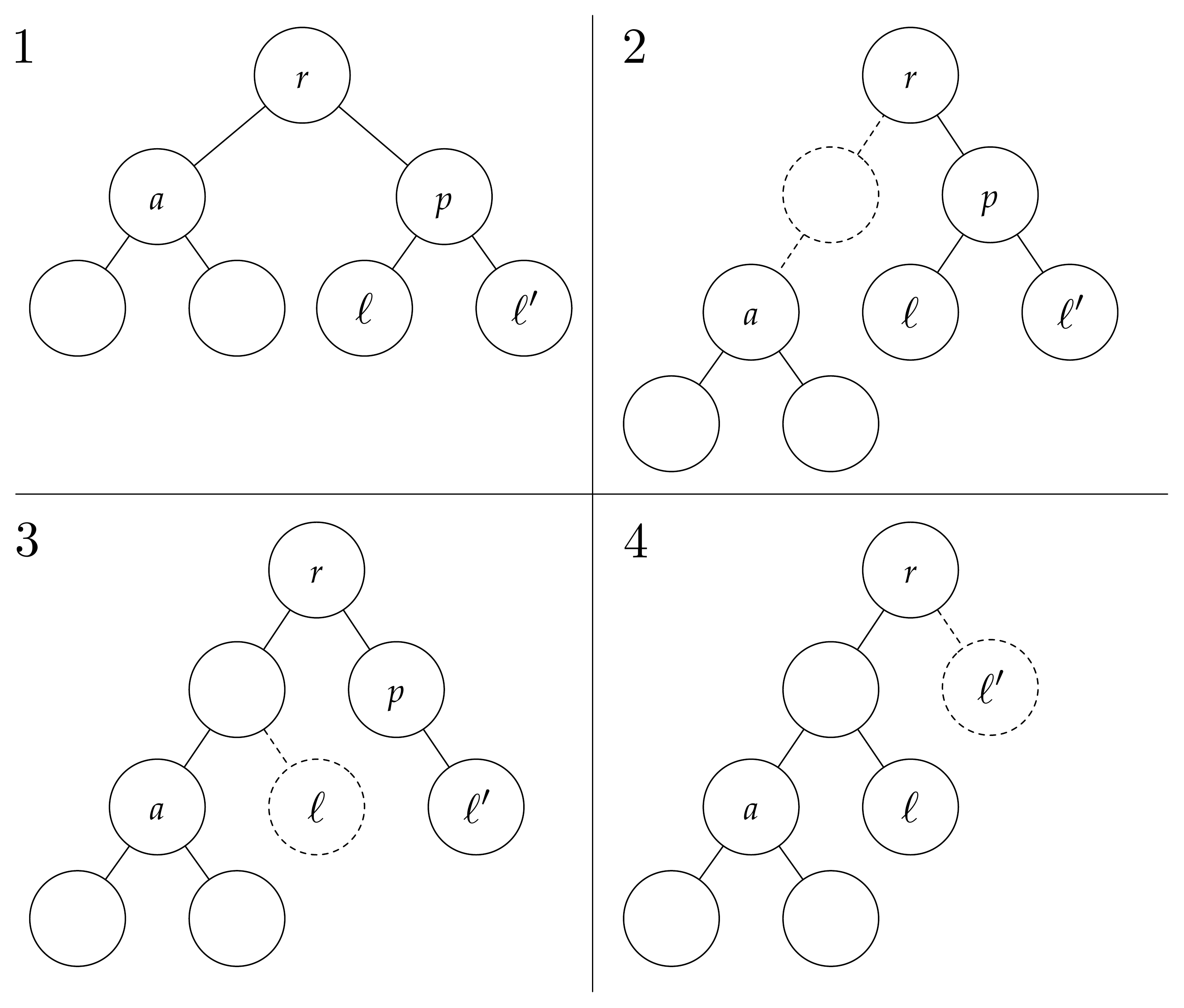}}
\caption{The \texttt{Rotate} procedure in four steps. In each
  step, the change to the tree is indicated by dotted lines.}
\label{fig:rotate}
\end{center}
%%\vskip -0.2in
\end{figure}

\begin{theorem}
\label{thm:rotate-opt}
  If $X$ is {separated} w.r.t. \cstar, the greedy algorithm with
  masking-based rotations constructs a cluster tree with dendrogram purity
  1.0.
\end{theorem}
\begin{proof}[Proof Sketch]
Inductively, assume our current cluster tree $\mathcal{T}$ has
dendrogram purity 1.0, and we process a new point $x_i$ belonging to
ground truth cluster $C^\star \in \cstar$.  In the first case, assume
that $\mathcal{T}$ already contains some members of $C^\star$ that are
located in a (pure) subtree, $\Tcal[C^\star]$. Then, by separability
$x_i$'s nearest neighbor must be in $\Tcal[C^\star]$ and if rotations
ensue, no internal node $v \in \Tcal[C^\star]$ can be rotated outside
of $T[C^\star]$. To see why, observe that the children of
$\Tcal[C^\star]$'s root cannot be masked since this is a pure subtree
and before insertion of $v$ the full tree itself was pure (so no
points from $C^\star$ can be outside of $\Tcal[C^\star]$).  In the
second case, assume that $T$ contains no points from cluster
$C^\star$. Then, again by separability, recursive rotations must lift
$x_i$ out of the pure subtree that contains its nearest neighbor,
which allows $\mathcal{T}$ to maintain perfect purity.
\end{proof}

	\section{Scaling}
\label{sec:scaling}
Two sub-procedures of the algorithm described above can render its
naive implementation slow: finding nearest neighbors and checking
whether a node is masked. Therefore, in this section we introduce
several approximations that make our final algorithm, \alg, scalable.
First, we describe a \emph{balance-based rotation} procedure that
helps to balance the tree and makes insertion of new points much
faster. Then we discuss a \emph{collapsed mode} that allows our
algorithm to scale to datasets that do not fit in memory. Finally, we
introduce a bounding-box approximation that makes both nearest
neighbor and masking detection operations efficient.

\subsection{Balance Rotations}
\label{subsec:balance}
While our rotation algorithm (Section \ref{subsec:mask-rot})
guarantees optimal dendrogram purity in the separable case, it does
not make any guarantees on the depth of the tree, which influences the
running time. Naturally, we would like to construct balanced binary
trees, as these lead to logarithmic time search and insertion. We use
the following notion of balance.

\begin{definition}[Cluster Tree Balance]
\label{def:balance}
The balance of a cluster tree $\Tcal$, denoted $\bal{\mathcal{T}}$, is
the average \emph{local balance} of all nodes in $\Tcal$, where the
local balance of a node $v$ with children $v_\ell,v_r$ is $\bal{v} =
\frac{\min\{|\lvs{v_\ell}|,|\lvs{v_r}|\}}{\max\{|\lvs{v_\ell}|,|\lvs{v_r}|\}}$.
\end{definition}

To encourage balanced trees, we use a balance-based rotation
operation. A balance-rotation with respect to a node $v$ with sibling
$v'$ and aunt $a$ is identical to a masking-based rotation with
respect to $v$, except that it is triggered when
\begin{itemize}
\item [1)] the rotation would produce a tree $\Tcal'$ with $\bal{\Tcal'}
  > \bal{\Tcal}$; and
\item [2)] there exists a point $x_i \in \lvs{v}$ such that $x_i$ is closer
  to a leaf of $a$ than to some leaf of $v'$
  (i.e. Equation~\eqref{eq:masking_def} holds).
\end{itemize}
Under separability, this latter check ensures that the balance-based
rotation does not compromise dendrogram purity.

\begin{fact}
\label{thm:balance-rot}
Let $X$ be a dataset that is separable w.r.t $\cstar$.  If $\Tcal$ is
a cluster tree with dendrogram purity 1 and $\Tcal'$ is the result of
a balance-based rotation on some node in $\Tcal$, then $\bal{\Tcal'}>
\bal{\Tcal}$ and $\Tcal'$ also has dendrogram purity 1.0.
\end{fact}

\begin{proof}
The only situation in which a rotation could impact the dendrogram
purity of a pure cluster tree $\Tcal$ is when $v$ and $v'$ belong to
the same cluster, but their aunt $a$ belongs to a different
cluster. In this case, separation ensures that
\begin{align*}
\max_{x \in \lvs{v}, y \in \lvs{v'}} \|x - y\| \le \min_{x' \in
  \lvs{v}, z \in \lvs{a}} \|x' - z\|,
\end{align*}
so a rotation will not be triggered.  Clearly, $\bal{\Tcal'} >
\bal{\Tcal}$ since it is explicitly checked before performing a
rotation.
\end{proof}

After inserting a new point and applying any masking-based rotations,
we check for balance-based rotations at each node along the path from
the newly created leaf to the root.

\subsection{Collapsed Mode}
\label{subsec:collapsing}
For extremely large datasets that do not fit in memory,
we use a \emph{collapsed mode} of \alg. In this mode, the algorithm
takes a parameter that is an upper bound on the number of leaves in the
cluster tree, which we denote with $L$.  After balance rotations, our
algorithm invokes a \texttt{Collapse} procedure that merges leaves as
necessary to meet the upper bound. This is similar to work in
hierarchical extreme classification in which the depth of the model is
a user-specified parameter~\cite{daume2016logarithmic}.

The \texttt{Collapse} operation may only be invoked on an internal
node $v$ whose children are both leaves. The procedure makes $v$ a
\emph{collapsed leaf} that stores the (ids, not the features, of the)
points associated with its children along with sufficient statistics
(Section \ref{subsec:bounding-box}), and then deletes both children.
The points stored in a collapsed leaf are
never split by any flavor of rotation. The \texttt{Collapse} operation
may be invoked on internal nodes whose children are collapsed leaves.

When the cluster tree has $L+1$ leaves, we collapse the node whose
maximum distance between children is minimal among all collapsible
nodes, and we use a priority queue to amortize the search for this
node. Collapsing this node is guaranteed to preserve dendrogram purity
in the separable case.

\begin{fact}
  \label{thm:collapse}
  Let $X$ be a dataset separable w.r.t. $\cstar$ which has $K$
  clusters, if $L > K$, then collapse operations preserve perfect
  dendrogram purity.
\end{fact}
\begin{proof}
  Inductively, assume that the cluster tree has dendrogram purity 1.0
  before we add a point that causes us to collapse a node. Since $L >
  K$ and all subtrees are pure, by the pigeonhole principle, there
  must be a pure subtree containing at least $2$ points. By
  separability, all such pure $2$-point subtrees will be at the front
  of the priority queue, before any impure ones, and hence the
  collapse will not compromise purity.
\iffalse
  Let $\mathcal{Q}$ be a queue of collapsible nodes such that
  $\mathcal{Q}$ contains some pure nodes, $\mathcal{C}^{+} =
  \{\mathcal{C}_{i} | \mathcal{C}_{i} \in \cstar\}$ and some impure
  nodes $\mathcal{C}^{-} = \{\mathcal{C}_{j} | \mathcal{C}_{j} \notin
  \cstar\}$. By separability, for any $\mathcal{C}_{i} \in
  \mathcal{C}^{+}$ and $\mathcal{C}_{j} \in \mathcal{C}^{-}$,
  \begin{align*}
    \max_{x, y \in \lvs{\mathcal{C}_{i}}} \|x - y\| \le \min_{x', y'
      \in \lvs{\mathcal{C}_{j}}} \|x' - z\|
  \end{align*}
  So, given the option, collapsing will always be applied to pure
  nodes before impure nodes. Since $L > K$, we are never forced to
  collapse if $\mathcal{Q}$ contains only impure nodes.
\fi
\end{proof}

\begin{algorithm}[t]
   \caption{\texttt{Insert}$(x_i, \Tcal)$}
   \label{alg:perch}
\begin{algorithmic}
   \STATE $t = $\texttt{ NearestNeighbor}$(x_i)$
   \STATE $l = $\texttt{ Split}$(t)$
   \FOR{$a$ in \texttt{Ancestors}$(l)$}
   \STATE $a.$\texttt{AddPt}$(x_i)$
   \ENDFOR
   \STATE $T = T.\texttt{RotateRec}(\ell.\texttt{Sibling}(), \texttt{CheckMasked})$
   \STATE $T = T.\texttt{RotateRec}(\ell.\texttt{Sibling}(), \texttt{CheckBalanced})$
   \IF{$\texttt{CollapseMode}$}
   \STATE $T.\texttt{TryCollapse()}$.
   \ENDIF
   \STATE {\bfseries Output:} $T$
\end{algorithmic}
\end{algorithm}

\begin{algorithm}[t]
   \caption{\texttt{RotateRec}$(v, \Tcal, \texttt{func})$}
   \label{alg:perch}
\begin{algorithmic}
   \STATE (\texttt{ShouldRotate,ShouldStop)} = $\texttt{func}(v)$
   \IF{$\texttt{ShouldRotate}$}
   \STATE $T.\texttt{Rotate}(v)$
   \ENDIF
   \IF{$\texttt{ShouldStop}$}
   \STATE {\bfseries Output:} $\Tcal$
   \ELSE
   \STATE {\bfseries Output:} $\Tcal.\texttt{RotateRec}(v.\texttt{Parent}(), \Tcal, \texttt{func})$
   \ENDIF
\end{algorithmic}
\end{algorithm}

\subsection{Bounding Box Approximations}
\label{subsec:bounding-box}
Many of the operations described thus far depend on nearest neighbor
searches, which in general can be computationally intensive. We
alleviate this complexity by approximating a set (or subtree) of
points via a bounding box. Here each internal node maintains a
bounding box that contains all of its leaves. Bounding boxes are easy
to maintain and update in $\mathcal{O}(d)$ time.

Specifically, for any internal node $v$ whose bounding interval in
dimension $j$ is $[v_-(j),v_+(j)]$, we approximate the squared minimum
distance between a point $x$ and $\lvs{v}$ by
\begin{align}
  d_{-}(x,v)^2 = \sum_{j=1}^d\left\{\begin{aligned}
  &(x(j) - v_{-}(j))^2 & \textrm{if } x(j) \le v_{-}(j)\\
  &(x(j) - v_{+}(j))^2 & \textrm{if } x(j) \ge v_{+}(j) \\
  &0 & \textrm{otherwise}
\end{aligned}\right. \label{eq:min_dist}
\end{align}
We approximate the squared maximum distance by
\begin{align}
d_+(x,v)^2 = \sum_{j=1}^d\max\left\{(x(j) - v_-(j))^2, (x(j)-v_+(j))^2\right\}. \label{eq:max_dist}
\end{align}
It is easy to verify that these provide lower and upper bounds on the
squared minimum and maximum distance between $x$ and $\lvs{v}$. See
Figure \ref{fig:bounding-box} for a visual representation.

For the nearest neighbor search involved in inserting a point $x$ into
$\mathcal{T}$, we use the minimum distance approximation $d_{-}(x,v)$
as a heuristic in $A^\star$ search. Our implementation maintains a
frontier of unexplored internal nodes of $\mathcal{T}$ and repeatedly
expands the node $v$ with minimal $d_{-}(x,v)$ by adding its children
to the frontier. Since the approximation $d_-$ is always a lower bound
and it is exact for leaves, it is easy to see that the first leaf
visited by the search is the nearest neighbor of $x$.  This is similar
to the nearest neighbor search in the Balltree data
structure~\cite{omohundro1989five}.

For masking rotations, we use a more stringent check based on
Equation~\ref{eq:masking_def}. Specifically, we perform a masking rotation
on node $v$ with sibling $v'$ and aunt $a$ if:

\begin{align}
d_{-}(v,v') > d_{+}(v,a)
\label{eq:bbmaskcheck}
\end{align}

Note that $d_{-}(v,v')$ is a slight abuse of notation (and so is $d_{+}(v,a)$)
because both $v$ and $v'$ are bounding boxes. To compute, $d_{-}(v,v')$,
for each dimension in $v$'s bounding box, use either the minimum or maximum
along that dimension to minimize the sum of coordinate-wise distances (Equation \ref{eq:min_dist}).
A similar procedure can be performed to compute $d_{+}(v,a)$ between two bounding boxes.

Note that if Equation \ref{eq:bbmaskcheck} holds then $v$ is masked and a rotation with respect to $v$ will unmask $v$ because, for all $x \in \lvs{v}$:
\begin{align*}
\max_{y \in \lvs{v'}} || x-y || &\geq \min_{y \in \lvs{v'}} || x-y || & \geq d_{-}(v,v') \geq d_{+}(v,a) & \geq \max_{z \in \lvs{a}} || x-z || & \geq 	\min_{z \in \lvs{a}} || x-z ||.
\end{align*}

In words, we know that rotation will unmask $v$ because the upper bound
on the distance from a point in $v$ to a point in $a$ is smaller than the lower
bound on the distance from a point in $v$ to a point in $v'$.

\begin{figure}[t]
\begin{center}
\centerline{\includegraphics[width=0.5\columnwidth]{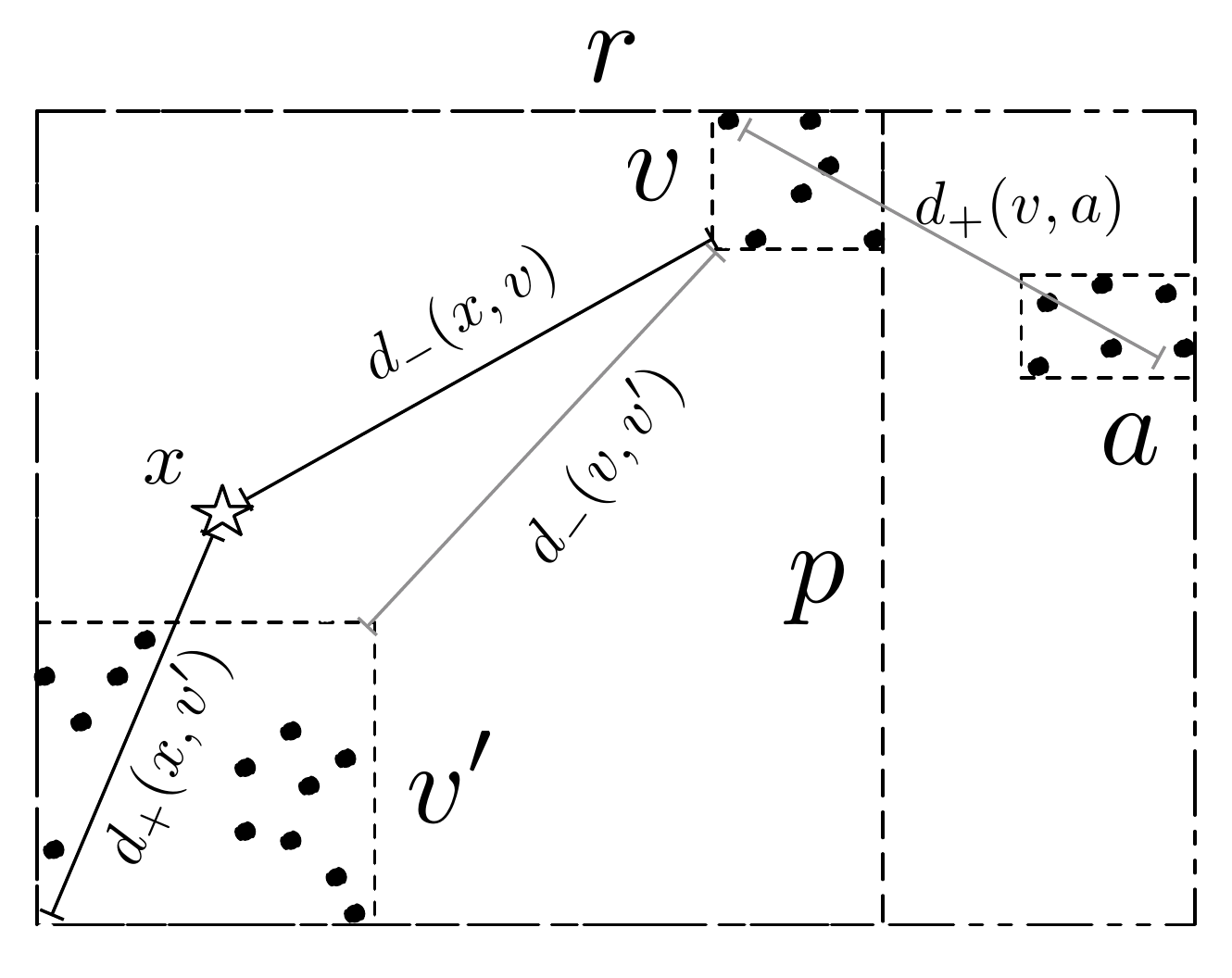}}
\caption{A subtree $r$ with two children $p$ and $a$; $p$ is the
  parent of $v$ and $v'$ (all nodes indicated by boxes with dashed
  outlines).  A point $x$ is inserted and descends to $v'$ because
  $d_{+}(x, v') < d_{-}(x, v)$ (black lines).  The node $v$ is masked
  because $d_{+}(v, a) < d_{-}(v, v')$ (gray lines).}
\label{fig:bounding-box}
\end{center}
\end{figure}

	\section{\alg}
Our algorithm is called \alg, for Purity Enhancing Rotations for
Cluster Hierarchies.

\alg can be run in several modes. All modes use bounding box
approximations, masking-, and balance-based rotations. The standard
implementation, simply denoted \alg, only includes these features (See
Algorithm~\ref{alg:perch}).  \algc additionally runs in collapsed mode
and requires the parameter $L$ for the maximum number of leaves. \algb
replaces the A$^\star$ search with a breadth-first beam search to
expedite the nearest neighbor lookup. This mode also takes the beam
width as a parameter. Finally \algbc uses both the beam search and
runs in collapsed mode.

We implement \alg to mildly exploit parallelism.  When finding nearest
neighbors via beam or A$^\star$ search, we use multiple threads to
expand the frontier in parallel.
	\section{Experiments}
\label{sec:exp}
We compare \alg to several state-of-the-art clustering algorithms on
large-scale real-world datasets. Since few benchmark clustering datasets
exhibit a large number of clusters, we conduct experiments with large-scale
classification datasets that naturally contain many classes and
simply omit the class labels. Note that other work in large-scale
clustering recognizes this deficiency with the standard clustering
benchmarks~\cite{betancourt2016flexible, ahmed2012fastex}. Although
smaller datasets are not our primary focus, we also measure the
performance of \alg on standard clustering benchmarks.

\subsection{Experimental Details}
\label{subsec:setup}
\paragraph{Algorithms.} We compare the following 10 clustering algorithms:
\begin{packed_item}
\item \alg\ - Our algorithm, we use several of the various
  modes depending on the size of the problem. For beam search, we use
  a default width of 5 (per thread, 24 threads). We only run in collapsed mode for
  ILSVRC12 and ImageNet (100K) where $L=50000$.
\item BIRCH - A top-down hierarchical clustering method where
  each internal node represents its leaves via mean and variance
  statistics. Points are inserted greedily using the node statistics,
  and importantly no rotations are performed.
\item HAC - Hierarchical agglomerative clustering (various
  linkages). The algorithm repeatedly merges the two subtrees that are
  closest according to some measure, to form a larger subtree.
\item Mini-batch HAC (MB-HAC) - HAC (various linkages) made to
  run online with mini-batching. The algorithm maintains a buffer of
  subtrees and must merge two subtrees in the buffer before observing
  the next point. We use buffers of size 2K and 5K and centroid (cent.) and
  complete (com.) linakges.
\item \km\ - LLoyd's algorithm, which alternates between assigning
  points to centers and recomputing centers based on the
  assignment. We use the \km++ initialization~\cite{arthur2007k}.
\item Stream \km++ (SKM++) \cite{ackermann2012streamkm++} - A streaming algorithm that computes a
  representative subsample of the points (a coreset) in one pass, runs
  \km++ on the subsample, and in a second pass assigns points
  greedily to the nearest cluster center.
\item Mini-batch \km\ (MB-KM) \cite{sculley2010web} - An algorithm that optimizes the \km\
  objective function via mini-batch stochastic gradient descent. The
  implementation we use includes many heuristics, including several
  initial passes through the data to initialize centers via
  \km++, random restarts to avoid local optima, and early
  stopping \cite{scikit-learn}. In a robustness experiments, we also study a
  fully online version of this algorithm, where centers are
  initialized randomly and without early stopping or random
  reassignments.
\item BICO \cite{fichtenberger2013bico} - An algorithm for optimizing
  the \km\ objective that creates coresets via a streaming approach
  using a BIRCH-like data structure. \km++ is run on the coresets and
  then points are assigned to the inferred centers.
\item DBSCAN \cite{ester1996density} - A density based method. The algorithm computes
  nearest-neighbor balls around each point, merges overlapping balls,
  and builds a clustering from the resulting connected components.
\item Hierarchical \km\ (HKMeans) - top-down, divisive, hierarchical
  clustering. At each level of the hierarchy, the remaining points are
  split into two groups using \km.
\end{packed_item}
These algorithms represent hierarchical and flat clustering approaches
from a variety of algorithm families including coreset, stochastic
gradient, tree-based, and density-based.  Most of the algorithms
operate online and we find that most baselines exploit parallelism
to various degrees, as we do with \alg.  We also compare to the less
scalable batch algorithms (HAC and \km) when possible.

\paragraph{Datasets.} We evaluate the algorithms on 9 datasets (See
Table~\ref{tbl:datasets} for relevant statistics):

\begin{packed_item}
\item ALOI - (Amsterdam Library of Object Images \cite{geusebroek2005amsterdam}) contains
  images and is used as an extreme classification
  benchmark~\cite{choromanska2015logarithmic}.
\item Speaker - The NIST I-Vector Machine Learning Challenge
  speaker detection dataset~\cite{dehak2011front}. Our
  goal is to cluster recordings from the same speaker together. We
  cluster the whitened development set (scripts provided by The
  Challenge).
\item ILSVRC12 - The ImageNet Large Scale Visual
  Recognition Challenge 2012
  \cite{ILSVRC15}. The class
  labels are used to produce a ground truth clustering. We generate
  representations of each image from the last layer of
  Inception\cite{szegedy2016rethinking}.
\item ILSVRC12 (50K) - a 50K subset of ILSVRC12.
\item ImageNet (100K) - a 100K subset of the ImageNet database. Image
  classes are sampled proportional to their frequency in the database.
\item Covertype - forest cover types (benchmark).
\item Glass - different types of glass (benchmark).
\item Spambase - email data of spam and not-spam
  (benchmark).
\item Digits - a subset of a handwritten digits dataset (benchmark).
\end{packed_item}
The Covertype, Glass, Spambase and Digits datasets are provided by the
UCI Machine Learning Repository~\cite{Lichman:2013}.

\begin{table}
\begin{tabular}{|c | c c c c|}
\hline
           & \bf Name        & \bf Clusters & \bf Points & \bf Dim.\\
\hline
           & ImageNet (100K) & 17K          & 100K       & 2048         \\
           & Speaker         & 4958         & 36,572     & 6388         \\
Large      & ILSVRC12        & 1000         & 1.3M       & 2048         \\
Data sets  & ALOI            & 1000         & 108K       & 128          \\
           & ILSVRC12 (50K)  & 1000         & 50K        & 2048         \\
           & CoverType       & 7            & 581K       & 54           \\
\hline
Small      & Digits          & 10           & 200        & 64           \\
Benchmarks & Glass           & 6            & 214        & 10           \\
           & Spambase        & 2            & 4601       & 57           \\
\hline
\end{tabular}
\caption{Dataset Statistics.}
\label{tbl:datasets}
\end{table}

\paragraph{Validation and Tuning.} As the data arrival order
impacts the performance of the online algorithms, we run each online
algorithm on random permutations of each dataset. We tune
hyperparameters for all methods and report the performance of the
hyperparameter with best average performance over 5 repetitions for the
larger datasets (ILSVRC12 and Imagenet (100K)) and 10 repetitions of the
other datasets.

\subsection{Hierarchical Clustering Evaluation}
\label{subsec:exphier}

\begin{figure*}[t]
  \captionsetup[subfigure]{justification=centering}
  %% \begin{table*}[t]
\begin{subfigure}[b]{1.0\textwidth}
	\footnotesize
\begin{tabular}{|c|c|c|c|c|c|c|}
	\hline
	\textbf{Method}               & \textbf{CovType} & \textbf{ILSVRC12 (50k)} & \textbf{ALOI}     & \textbf{ILSVRC 12}  & \textbf{Speaker}   & \textbf{ImageNet (100k)} \\
%	\textbf{}               & \textbf{} & \textbf{(50k)} & \textbf{}     & \textbf{}  & \textbf{}   & \textbf{(100k)} \\

%	                              & $K$=7,$N$=500k     & $K$=1k,$N$=50k     & $K$=1k,$N$=108k     & $K$=1k,$N$=1.3M  & $K$=4.9k,$N$=36k   & $K$=17k,$N$=100k \\
	\hline
	\textbf{\alg}                 & \bf 0.45 $\pm$ 0.004   & \bf 0.53 $\pm$ 0.003  & \bf 0.44 $\pm$ 0.004   & ---               & \bf 0.37 $\pm$ 0.002   & \bf  0.07 $\pm$ 0.00         \\
	\textbf{\algbc}               & 0.45 $\pm$ 0.004   & 0.36 $\pm$ 0.005         & 0.37 $\pm$ 0.008   & \bf 0.21 $\pm$ 0.017   & 0.09 $\pm$ 0.001   &   0.03 $\pm$ 0.00         \\
	\textbf{BIRCH (online)}       & 0.44 $\pm$ 0.002   & 0.09 $\pm$ 0.006         & 0.21 $\pm$ 0.004   & 0.11 $\pm$ 0.006  & 0.02 $\pm$ 0.002   &   0.02 $\pm$ 0.00       \\
	\textbf{MB-HAC-Com.}      &  ---               & 0.43 $\pm$ 0.005         & 0.15 $\pm$ 0.003   & ---                & 0.01 $\pm$ 0.002   &  ---                       \\
	\textbf{MB-HAC-Cent.}      & 0.44 $\pm$ 0.005   & 0.02 $\pm$ 0.000         & 0.30 $\pm$ 0.002   & ---                & ---                &  ---                      \\
	\textbf{HKMmeans}             & 0.44 $\pm$ 0.001   &  0.12 $\pm$ 0.002        & \bf 0.44 $\pm$ 0.001   & 0.11 $\pm$ 0.003   & 0.12 $\pm$ 0.002   & 0.02 $\pm$ 0.00     \\
    \textbf{BIRCH (rebuild)     }     & 0.44 $\pm$ 0.002    &  0.26 $\pm$ 0.003        & 0.32 $\pm$ 0.002   & ---                & 0.22 $\pm$ 0.006   &   0.03 $\pm$ 0.00 \\
%	\textbf{HA-Avg}              & ---                 &  0.54                    &  ---               & ---                & 0.55               &                     \\
%	\textbf{HA-Complete}         & ---                 &  0.40                    &  ---               &  ---               & 0.40               &                     \\
	\hline
\end{tabular}
\caption{Dendrogram Purity for Hierarchical Clustering.}
\label{tbl:dp_results}
\end{subfigure}
%% \end{table*}
\\
  \vspace{5mm}
  \begin{subfigure}[h]{1.0\textwidth}
  \centering
  \footnotesize
\begin{tabular}{|c|c|c|c|c|c|c|}
	\hline
	\textbf{Method}          & \textbf{CoverType}    & \textbf{ILSVRC 12 (50k)} & \textbf{ALOI}       & \textbf{ILSVRC 12}    & \textbf{Speaker}      & \textbf{ImageNet (100K)}   \\
	\hline
	\textbf{\alg}            & 22.96 $\pm$ 0.7       & \bf 54.30 $\pm$ 0.3      & \bf 44.21 $\pm$ 0.2 &  ---                  & \bf 31.80 $\pm$ 0.1   & \bf 6.178 $\pm$ 0.0        \\
	\textbf{\algbc}          & 22.97 $\pm$ 0.8       & 37.98 $\pm$ 0.5          &  37.48 $\pm$ 0.7    & 25.75 $\pm$ 1.7       & 1.05 $\pm$ 0.1        &  4.144 $\pm$ 0.04          \\
	\textbf{SKM++}           & 23.80 $\pm$ 0.4       & 28.46 $\pm$ 2.2          & 37.53 $\pm$ 1.0     &  ---                  & ---                   &     ---                    \\
	\textbf{BICO} & \bf 24.53 $\pm$ 0.4 & 45.18 $\pm$ 1.0  & 32.984 $\pm$ 3.4 & --- & --- & --- \\
	\textbf{MB-KM}           &  24.27 $\pm$ 0.6   & 51.73 $\pm$ 1.8          & 40.84 $\pm$ 0.5     & \bf 56.17 $\pm$ 0.4   & 1.73 $\pm$ 0.141      &  5.642 $\pm$ 0.00          \\
	\textbf{DBSCAN}          &  ---                  & 16.95                    & ---                 & ---                  & 22.63                  &  ---                       \\
	\hline
\end{tabular}
\caption{Pairwise F1 for Flat Clustering.}
\label{tbl:pw_results}
\end{subfigure}
\\
  \vspace{5mm}
  \begin{subfigure}[h]{0.49\textwidth}
    \centering
    \begin{tabular}{|c| c c |}
\hline
\bf       & Round.  & Sort. \\
\hline
\alg      & 44.77   &  35.28   \\
o-MB-KM   & 41.09   &  19.40   \\
SKM++     & 43.33   &  46.67   \\
\hline
\end{tabular}
\caption{Pairwise F1 on adversarial input orders for \\ALOI.}
\label{tab:robust-f1}

  \end{subfigure}
  \begin{subfigure}[h]{0.49\textwidth}
    \centering
    \begin{tabular}{|c| c c |}
\hline
\bf         & Round. & Sort. \\
\hline
\alg        & 0.446  &  0.351 \\
%%MB-HAC (3K) & 0.234  &  0.458 \\
MB-HAC (5K) & 0.299  &  0.464 \\
MB-HAC (2K) & 0.171  &  0.451 \\
\hline
\end{tabular}
\caption{Dendrogram Purity on adversarial input orders
  for ALOI.}
\label{tab:robust-dp}

  \end{subfigure}
  \caption{\alg is the top performing algorithm in terms of dendrogram
    purity competitive in F1. \alg is nearly twice as fast as MB-KM on
    ImageNet (100K) (Section \ref{subsec:expbirch}). Dashes
    represent algorithms that could not be run or produced low results.}
\end{figure*}

\alg, and many of the other baselines build cluster trees, and as such
they can be evaluated using dendrogram purity. In
Table~\ref{tbl:dp_results}, we report the dendrogram purity, averaged
over random shufflings of each dataset, for the 6 large datasets,
and for the scalable hierarchical clustering algorithms (BIRCH, MB-HAC
variants, and HKMeans). The top 5 rows in the table correspond to online
algorithms, while the bottom 2 are classical batch algorithms. The
comparison demonstrates the quality of the online algorithms, as well
as the degree of scalability of the offline methods. Unsurprisingly,
we were not able to run some of the batch algorithms on the larger
datasets; these algorithms do not appear in the table.

\alg consistently produces trees with highest dendrogram purity
amongst all online methods. \algb, which uses an approximate
nearest-neighbor search, is worse, but still consistently better than
the baselines. For the datasets with 50K examples or fewer (ILSVRC12
50K and Speaker) we are able to run the less scalable algorithms. We
find that HAC with average-linkage achieves a dendrogram purity of
0.55 on the Speaker dataset and only outperforms \alg by 0.01
dendrogram purity on ILSVRC12 (50K). HAC with complete-linkage only
slightly outperforms \alg on Speaker with a dendrogram purity of 0.40.

We hypothesize that the success of \alg in these experiments can
largely be attributed to the rotation operations and the bounding box
approximations. In particular, masking-based rotations help alleviate
difficulties with online processing, by allowing the algorithm to make
corrections caused by difficult arrival orders. Simultaneously, the
bounding box approximation is a more effective search heuristic for
nearest neighbors in comparison with using cluster means and variances
as in BIRCH. We observe that MB-HAC with centroid-linkage
performs poorly on some datasets, which is likely due to the fact that
a cluster's mean can be an uninformative representation of its member
points, especially in the online setting.

\subsection{Flat Clustering Evaluation}
\label{subsec:expflat}
We also compare \alg to the flat-clustering algorithms described
above. Here we evaluate a $K$-way clustering via the Pairwise F1
score~\cite{manning2008introduction}, which given ground truth
clustering $\cstar$ and estimate $\chat$, is the harmonic mean of the
precision and recall between $\pstar$ and $\phat$ (which are pairs of
points that are clustered together in $\cstar,\chat$ respectively). In
this section we compare $\alg$ to MB-KM, SKM++, and DBSCAN. On the
smaller datasets, we also compare to the offline Lloyd's algorithm for
$k$-means clustering (with \km++ initialization) and HAC where the
tree construction terminates once $K$-subtrees remain. All algorithms,
including \alg, use the true number of clusters as an input parameter
(except DBSCAN, which does not take the number of clusters as input).

Since \alg produces a cluster tree, we extract a tree-consistent
partition using the following greedy heuristic: while the tree
$\mathcal{T}$ does not have $K$ leaves, collapse the internal node
with the smallest \emph{cost}, where $\textrm{cost}(v)$ is the maximum
length diagonal of $v$'s bounding box multiplied by
$|\lvs{v}|$. $\textrm{cost}(v)$ can be thought of as an upper bound on
the $K$-means cost of $v$.  We also experimented with other
heuristics, but found this one to be quite effective.

The results of the experiment are in Table~\ref{tbl:pw_results}.  As
in the case of dendrogram purity, \alg competes with or outperforms
all the scalable algorithms on all datasets, even though we use a
na\"{i}ve heuristic to identify the final flat clustering from the
tree.  \km, which could not be run on the larger datasets, is able
to achieve a clustering with 60.4 F1 on ILSVRC (50K) and 32.19 F1 on
Speaker; HAC with average-linkage and HAC with complete-linkage
achieve scores of 40.26 and 44.3 F1 respectively on Speaker. This is unsurprising
because in each iteration of both \km\ and HAC, the algorithm has
access to the entire dataset. We were able to
run \km\ on the Covertype dataset and achieve a score of 24.42
F1. We observe that \km\ is able to converge quickly, which is
likely due to the \km++ initialization and the small number of
true clusters (7) in the dataset. Since \alg performs well on each of
the datasets in terms of dendrogram purity, it could be possible to
extract better flat clusterings from the trees it builds with a
different pruning heuristic.

We note that when the number of clusters is large, DBSCAN does not
perform well, because it (and many other density based algorithms)
assumes that some of the points are \emph{noise} and keeps them
isolated in the final clustering. This outlier detection step is
particularly problematic when the number of clusters is large, because
there are many small clusters that are confused for outliers.

\subsection{Speed and Accuracy}
\label{subsec:expbirch}
We compare the best performing methods above in terms of running time
and accuracy. We focus \algbc, BIRCH, and HKMeans. Each of these
algorithms has various parameter settings that typically govern a
tradeoff between accuracy and running time, and in our experiment, we
vary these parameters to better understand this tradeoff. The specific
parameters are:
\begin{packed_item}
\item \textbf{\algbc}: beam-width, collapse threshold,
\item \textbf{BIRCH}: branching factor,
\item \textbf{HKMeans}: number of iterations per level.
\end{packed_item}

\begin{figure}[t]
\begin{center}
\centerline{\includegraphics[width=0.75\columnwidth]{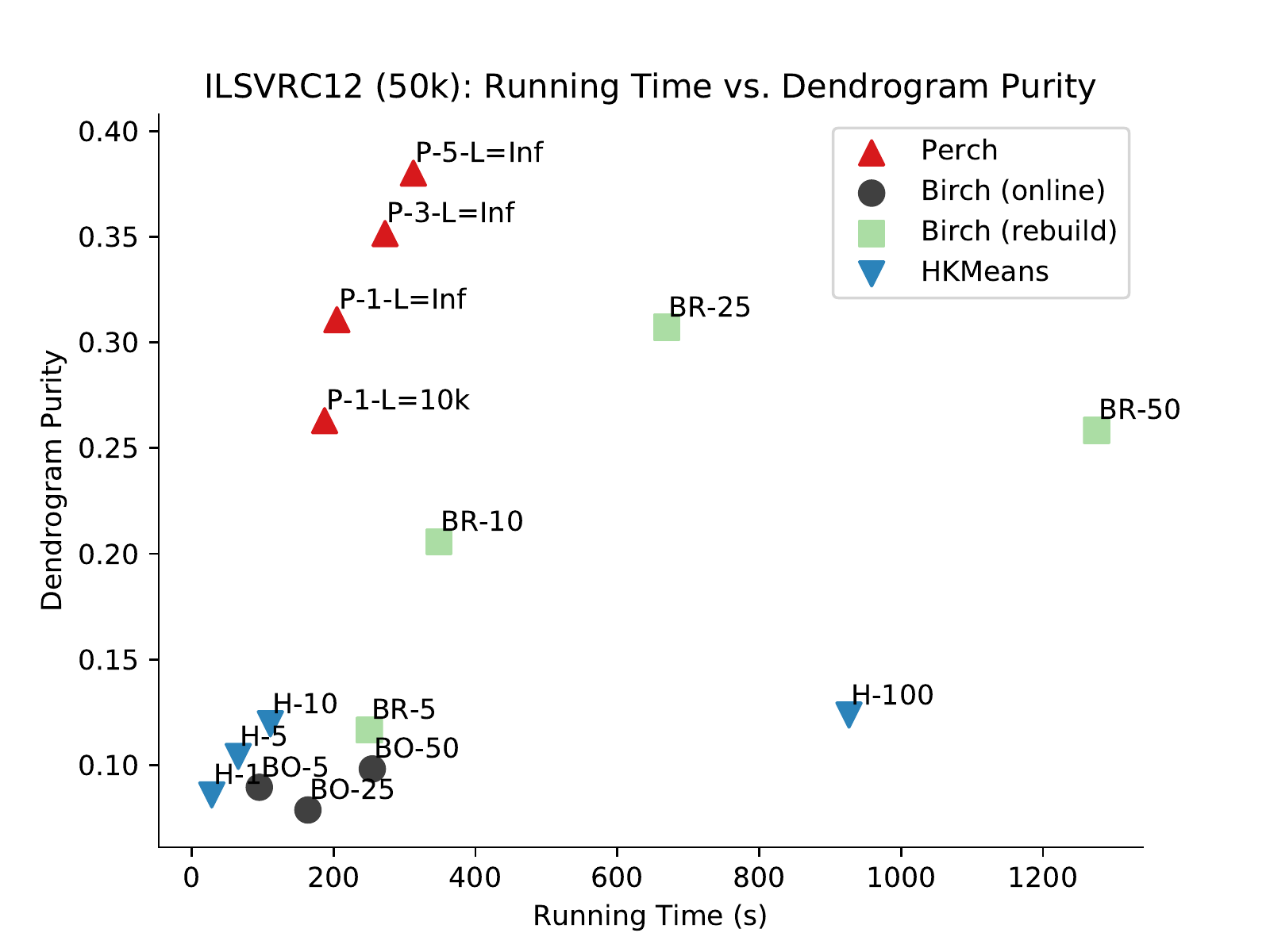}}
\caption{Speed vs. dendrogram purity on the ILSVRC12 50k dataset. \alg (P)
  is always more accurate than BIRCH (B) even when tuned so the
  running times are comparable. HKmeans (H) is also depicted for
  comparison against a fast, batch, tree building algorithm.}
\label{fig:speed-vs-accuracy}
\end{center}
\vskip -0.4in
\end{figure}

In Figure~\ref{fig:speed-vs-accuracy}, we plot the dendrogram purity
as a function of running time for the algorithms as we vary the
relevant parameter. We use the ILSVRC12 (50K) dataset and run all
algorithms on the same machine (28 cores with 2.40GHz
processor).

The results show that \algbc achieves a better tradeoff between
dendrogram purity and running time than the other methods. Except for
in the uninteresting low purity regime, our algorithm achieves the
best dendrogram purity for a given running time. HKMeans with few
iterations can be faster, but it produces poor clusterings, and with
more iterations the running time scales quite poorly. BIRCH with rebuilding
performs well, but \algbc performs better in less time.

Our experiments also reveal that MB-KM is quite effective, both in
terms of accuracy and speed. MB-KM is the fastest algorithm when the
number of clusters is small while \alg is fastest when the number of
clusters is very large. When the number of clusters is fairly small,
MB-KM achieves the best speed-accuracy tradeoff.  However, since the
gradient computation in MB-KM are $O(K)$, the algorithm scales poorly
when the number of clusters is large, while our algorithm, provided
short trees are produced, has no dependence on $K$. For example, for
ImageNet (100K) (the dataset with the largest number of clusters),
MB-KM runs in $\sim$8007.93s (averaged over 5 runs) while \alg runs
nearly twice as fast in 4364.37s.  Faster still (although slightly
less accurate) is \algbc which clusters the dataset in only 690.16s.

In Figure~\ref{fig:impure-by-k}, we confirm empirically that insertion
times scale with the size of the dataset. In the figure, we plot the
insertion time and the maximum tree depth as a function of the number
of data points on ILSVRC12 (50K). We also plot the same statistics for
a variant of \alg that does not invoke balance-rotations to understand
how tree structure affects performance.

\begin{figure}[ht]
\vskip 0.2in
\begin{center}
\centerline{\includegraphics[width=0.75\columnwidth]{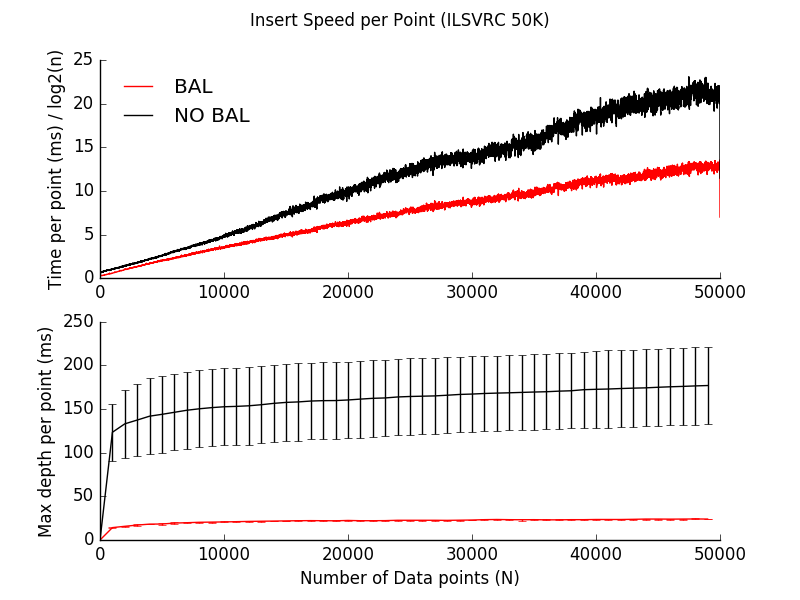}}
\caption{Insert speed divided by $\log_2(n)$ and max-depth as a
  function of the number of points inserted. \alg is the red (bottom)
  line and \alg without balance-rotations is black (top).  Balancing
  helps to make insertion fast by keeping the tree short.}
\label{fig:impure-by-k}
\end{center}
\end{figure}

The top plot shows that \alg's speed of insertion grows modestly with
the number of points, even though the dimensionality of the data is
high (which can make bounding box approximations worse).  Both the top
and bottom plots show that \alg's balance rotations have a significant
impact on its efficiency; in particular, the bottom plot suggests that
the efficiency can be attributed to shallower trees. Recall that our
analysis does not bound the depth of the tree or the running time of
the nearest neighbor search. Thus in the worst case the algorithm may
explore all leaves of the tree and have total running time scaling with
$O(N^2d)$ where there are $N$ data points. Empirically we observe that
this is not the case.

\subsection{Robustness}
We are also interested in understanding the performance of \alg and
other online methods as a function of the arrival order. \alg is
designed to be robust to adversarial arrival orders (at least under
separability assumptions), and this section empirically validates this
property. On the other hand, many online or streaming implementations
of batch clustering algorithms can make severe irrecoverable mistakes
on adversarial data sequences.  Our (non-greedy) masking-rotations
explicitly and deliberately alleviate such behavior.

To evaluate robustness, we run the algorithms on two adversarial
orderings of the ALOI dataset:
\begin{itemize}
  \item \textbf{Sorted} - the points are ordered by class (i.e., all
     $x \in C_k^\star \subset \cstar$ are followed by all $x \in
    C_{i+1}^\star \subset \cstar$, etc.)
  \item \textbf{Round Robin} - the $i^{th}$ point to arrive is a
    member of $C_{i\bmod K}^\star$ where $K$ is the true number of
    clusters.
\end{itemize}
The results in Tables~\ref{tbl:dp_results} and~\ref{tbl:pw_results}
use a more benign random order.

Tables~\ref{tab:robust-dp} and \ref{tab:robust-f1} contain the results
of the robustness experiment.  \alg performs best on
round-robin. While \alg's dendrogram purity decreases on the sorted
order, the degradation is much less than the other methods. Online
versions of agglomerative methods perform quite poorly on round-robin
but much better on sorted orders. This is somewhat surprising, since
the methods use a buffer size that is substantially larger than the
number of clusters, so in both orders there is always a good merge to
perform. For flat clusterings, we compare \alg with an online version
of mini-batch \km\ (o-MB-KM).  This algorithm is restricted to pick
clusters from the first mini-batch of points and not allowed to drop
or restart centers.  The o-MB-KM algorithm has significantly worse
performance on the sorted order, since it cannot recover from
separating multiple points from the same ground truth cluster. It is
important to note the difference between this algorithm and the MB-KM
algorithm used elsewhere in these experiments, which is robust to the
input order. SKM++ improves since it uses a non-greedy method for the
coreset construction. However, SKM++ is not technically an online
algorithm, since it makes two passes over the dataset.

\subsection{Small-scale Benchmarks}
\label{subsec:expsmall}
Finally, we evaluate our algorithm on the standard (small-scale)
clustering benchmarks.  While \alg is designed to be accurate and
efficient for large datasets with many clusters, these results help
us understand the price for scalability, and provide a more exhaustive
experimental evaluation. The results appear in
Table~\ref{tab:small-scale} and show that \alg is competitive with
other batch and online algorithms (in terms of dendrogram purity),
despite only examining each point once, and being optimized for large-scale problems.

\begin{figure}[h]
  \vskip 0.2in
  \centering
  \begin{tabular}{|c| c c c |}
    \hline
    \bf     & Glass             & Digits            & Spambase \\
    \hline
    \alg    & 0.474 $\pm$ 0.017 & 0.614 $\pm$ 0.033 & 0.611 $\pm$ 0.0131 \\
    BIRCH   & 0.429 $\pm$ 0.013 & 0.544 $\pm$ 0.054 & 0.595 $\pm$ 0.013 \\
    \hline
    HKMeans & 0.508 $\pm$ 0.008 & 0.586 $\pm$ 0.029 & 0.626 $\pm$ 0.000  \\
    HAC-C   & 0.47              & 0.594             & 0.628 \\
    \hline
  \end{tabular}
  \caption{Dendrogram purity on small-scale benchmarks.}
  \label{tab:small-scale}
\end{figure}

	\section{Related Work}
The literature on clustering is too vast for an in-depth treatment
here, so we focus on the most related methods. These can be
compartmentalized into hierarchical methods, online optimization of
various clustering cost functions, and approaches based on
coresets. We also briefly discuss related ideas in supervised learning
and nearest neighbor search.

Standard hierarchical methods like single linkage are often the
methods of choice for small datasets, but, with running times scaling
quadratically with sample size, they do not scale to larger
problems. As such, several online hierarchical clustering algorithms
have been developed. A natural adaptation of these agglomerative
methods is the mini-batch version that \alg outperforms in our
empirical evaluation. BIRCH~\cite{zhang1996birch} and its
extensions~\cite{fichtenberger2013bico} comprise state of the art online
hierarchical methods that, like \alg, incrementally insert points into
a cluster tree data structure. However, unlike \alg, these methods
parameterize internal nodes with means and variances as opposed to
bounding boxes, and they do not implement rotations, which our
empirical and theoretical results justify. On the other hand,
Widyantoro et al.~\cite{widyantoro2002incremental} instead use a
bottom-up approach.

A more thriving line of work focuses on incremental optimization of
clustering cost functions. A natural approach is to use stochastic
gradient methods to optimize the \km\
cost~\cite{bottou1995convergence,sculley2010web}. Liberty et
al.~\cite{liberty2016algorithm} design an alternative online \km\
algorithm that when processing a point, opts to start a new cluster if
the point is far from the current centers. This idea draws inspiration
from the algorithm of Charikar et al.~\cite{charikar1997incremental}
for the online $k$-center problem, which also adjusts the current
centers when a new point is far away. Closely related to these
approaches are several online methods for inference in a probabilistic
model for clustering, such as stochastic variational methods for
Gaussian Mixture Models~\cite{hoffman2013stochastic}.  As our
experiments demonstrate, this family of algorithms is quite effective
when the number of clusters is small, but for problems with many
clusters, these methods typically do not scale.

Lastly, a number of clustering methods use coresets, which are small
but representative data subsets, for scalability. For example, the
StreamKM++ algorithm of Ackermann et al.~\cite{ackermann2012streamkm++}
and the BICO algorithm of Fichtenberger et al.
~\cite{fichtenberger2013bico}, run the \km++ algorithm on a
coreset extracted from a large data stream. Other methods with strong
guarantees also exist~\cite{badoiu2002approximate}, but typically
coreset construction is expensive, and as above, these methods do not
scale to the extreme clustering setting where $K$ is large.

While not explicitly targeting the clustering task, tree-based methods
for nearest neighbor search and extreme multiclass classification
inspired the architecture of \alg. In nearest neighbor search, the
cover tree structure~\cite{beygelzimer2006cover} represents points
with a hierarchy while supporting online insertion and deletion, but
it does not perform rotations or other adjustments that improve
clustering performance. Tree-based methods for extreme classification
can scale to a large number of classes, but a number of algorithmic
improvements are possible with access to labeled data. For example,
the recall tree~\cite{daume2016logarithmic} allows for a class to be
associated with many leaves in the tree, which does not seem possible
without supervision.

	\section{Conclusion}
In this paper, we present a new algorithm, called \alg, for
large-scale clustering. The algorithm constructs a cluster tree in an
online fashion and uses rotations to correct mistakes and encourage a
shallow tree. We prove that under a separability assumption, the
algorithm is guaranteed to recover a ground truth clustering and we
conduct an exhaustive empirical evaluation. Our experimental results
demonstrate that \alg outperforms existing baselines, both in terms of
clustering quality and speed. We believe these experiments
convincingly demonstrate the utility of \alg.

Our implementation of \alg used in these experiments is available at: \url{http://github.com/iesl/xcluster}.

	\subsection*{Acknowledgments}
We thank Ackermann et. al~\cite{ackermann2012streamkm++} for providing
us an implementation of StreamKM++ and BICO, and Luke Vilnis for many
helpful discussions.

This work was supported in part by the Center for Intelligent
Information Retrieval, in part by DARPA under agreement number
FA8750-13-2-0020, in part by Amazon Alexa Science, in part by the National 
Science Foundation Graduate
Research Fellowship under Grant No. NSF-1451512 and in part by the
Amazon Web Services (AWS) Cloud Credits for Research program. The work
reported here was performed in part using high performance computing
equipment obtained under a grant from the Collaborative R\&D Fund
managed by the Massachusetts Technology Collaborative. Any opinions,
findings and conclusions or recommendations expressed in this material
are those of the authors and do not necessarily reflect those of the
sponsor.

	\bibliographystyle{ACM-Reference-Format}
	\bibliography{clustering}

\end{document}